\documentclass[twoside,12pt]{article}
\usepackage{algorithm}
\usepackage{algorithmic}
\usepackage{amsmath}
\usepackage{amssymb}
\usepackage{amsthm}
\usepackage{enumitem}
\usepackage{graphicx}
\usepackage{xcolor}
\usepackage{booktabs}
\usepackage{blindtext}

\usepackage{jmlr2e}

\definecolor{citecol}{HTML}{6F130C}
\definecolor{tableofcontent}{HTML}{1F4A83}
\definecolor{urlcol}{HTML}{2470D8}

\usepackage{hyperref}
\hypersetup{
    colorlinks=true,       
    linkcolor=tableofcontent,
    citecolor=citecol,        
    urlcolor=blue,           
}

\usepackage{cleveref}

\usepackage{lastpage}

\firstpageno{1}

\begin{document}

\title{How Particle-System Random Batch Methods Enhance Graph Transformer: Memory Efficiency and Parallel Computing Strategy}

\author{\name Hanwen Liu
\email \texttt{\rm jiaoshuangxi@sjtu.edu.cn}\\
       \addr Shanghai Jiao Tong University
    \AND
\name  Yixuan Ma
\email \texttt{\rm mayx5901@sjtu.edu.cn}\\
       \addr Shanghai Jiao Tong University
    \AND
\name Shi Jin
\email \texttt{\rm shijin-m@sjtu.edu.cn}\\
       \addr 
       Shanghai Jiao Tong University
    \AND
\name Yu Guang Wang\thanks{Corresponding author} 
\email \texttt{\rm yuguang.wang@sjtu.edu.cn}\\
       \addr 
       Shanghai Jiao Tong University}

\editor{My editor}

\maketitle

\begin{abstract}
Attention mechanism is a significant part of Transformer models. It helps extract features from embedded vectors by adding global information and its expressivity has been proved to be powerful. Nevertheless, the quadratic complexity restricts its practicability. Although several researches have provided attention mechanism in sparse form, they are lack of theoretical analysis about the expressivity of their mechanism while reducing complexity. In this paper, we put forward Random Batch Attention (RBA), a linear self-attention mechanism, which has theoretical support of the ability to maintain its expressivity.
Random Batch Attention has several significant strengths as follows: 
(1) Random Batch Attention has linear time complexity. Other than this, it can be implemented in parallel on a new dimension, which contributes to much memory saving.
(2) Random Batch Attention mechanism can improve most of the existing models by replacing their attention mechanisms, even many previously improved attention mechanisms.
(3) Random Batch Attention mechanism has theoretical explanation in convergence, as it comes from Random Batch Methods on computation mathematics.
Experiments on large graphs have proved advantages mentioned above.
Also, the theoretical modeling of self-attention mechanism is a new tool for future research on attention-mechanism analysis.
\end{abstract}

\begin{keywords}
  Attention mechanism, Random Batch Methods, Graph transformer 
\end{keywords}

\section{Introduction}


Over the past decade, Transformers ~\citep{vaswani2017attention} have demonstrated groundbreaking performance across diverse domains, most notably in natural language processing, as well as in computer vision \cite{dosovitskiy2021imageworth16x16words} and graph signal processing \cite{min2022Transformer}.
Especially, in recent several years, large language models comes into being and achieved unprecedented success on many fields.
For example, OpenAI's GPT-4 \cite{openai2024gpt4technicalreport} emerged as the first widely-accessible multi-modal LLM, demonstrating superior reasoning capabilities and setting new benchmarks in tasks ranging from complex language understanding to coding challenges. Simultaneously, specialized applications like MEDIC \cite{0Large} showcased LLMs' transformative potential in healthcare by significantly reducing medication errors in pharmacy settings, and Imagen \cite{saharia2022photorealistictexttoimagediffusionmodels} achieves photorealistic text-to-image generation by cascading diffusion models with a frozen T5-XXL language encoder, enabling fine-grained semantic alignment between text and pixels. These developments collectively highlight the rapid evolution of LLMs toward greater multi-modal capability, specialized domain expertise, and computational efficiency.

Attention mechanism is a crutial part of Transformers. Transformers completely abandon traditional Recurrent Neural Networks (RNNs) and Convolutional Neural Networks (CNNs), relying instead on the self-attention mechanism to capture dependencies between different positions in the input sequence. The self-attention mechanism calculates the relevance of each element in the input sequence to every other element, generating a weight matrix used to compute a weighted sum, resulting in a contextual representation for each element. This allows the Transformer to process the entire sequence in parallel while capturing long-range dependencies, significantly improving both efficiency and performance.

Specifically, the multi-head attention mechanism in Transformers further extends the capabilities of self-attention. By projecting the input into multiple subspaces and computing attention separately, the model can capture complex patterns in the sequence from different perspectives. This design not only enhances the model's expressive power but also makes it more flexible in handling diverse tasks.

The sparse design of attention mechanism is a popular research in recent several years. In 2019, \cite{child2019generatinglongsequencessparse} put forward the Sparse Transformer, change the traditional attention mechanism with sparse factorizations of the attention matrix, reduce the time complexity to $O(N^{\frac{3}{2}})$. In 2020, \cite{kitaev2020reformerefficientTransformer} utilize Hash function to map query and key to sparsify the two matrix, reduce the time complexity to $O(N\mathrm{log}N)$. In 2020, \cite{NEURIPS2020_c8512d14} introduced a sparse attention mechanism that combines global, local, and random attention patterns, significantly reducing computational complexity to $O(N)$ while maintaining the ability to capture both long-range and local dependencies in long sequences. In 2023, \cite{ding2023longnetscalingTransformers1000000000} replace the query, key and value with their sparsified segments, further reduce the time complexity to $O(N)$. All the researches proved the expressivity of their module through experiments on different datasets in the field of natural language processing.


In the field of graph signal processing, the most popular artificial intelligence tool is graph neural networks, most of which are based on graph convolution networks.
\cite{kipf2017semisupervisedclassificationgraphconvolutional}. However, due to the over-smoothing phenomenon, graph convolution networks can not pile with many layers. Other than this, the adjacent factor hinder neural networks from extract the global information.

As illustration, \cite{yun2019graph,zhao2021gophormer, wu2023difformer, wu2022nodeformer, wu2024simplifying, deng2024polynormer, zhang2020graph, wu2021representing, ying2021Transformers, rampavsek2022recipe, ma2023graph} take graph data inputs as prompts and utilize Transformer as an encoder.
Many of these architectures add graph structure information into the construction of graph transformers. One prevalent methodology is to design a structure encoding (SE) or position encoding (PE) and send it to the attention calculation module as a bias \cite{ying2021Transformers, luo2022one, rampavsek2022recipe}, which has been proved to be an excellent idea similar as PE in natural language processing. 

Despite the good expressivity of graph transformers, the practicability is always restrained by the quadratic time complexity of the attention mechanism and the non-parallelism of large graph data without batches but difficult to gain representation in a short time and with little memory. To lower the time complexity, several works have tried to sparse the attention mechanism or simplify the calculation steps. For example, \citet{rampavsek2022recipe} directly make use of the linear Transformers, Performer \cite{choromanski2020rethinking} and BigBird \cite{zaheer2020big} models to construct its architecture. \cite{wu2024simplifying} and \cite{deng2024polynormer} change the order in attention calculation steps, multiply Key and Value in former and multiple the result with Query in latter to reduce the complexity.
 
Nonetheless, most of these scalable graph transformers have no detailed mathematical analysis on the expressivity, which is a powerful support along with the experiment results. Other than this, in large graph downstream tasks, they are only practiced in single device, affecting the efficiency of their practical applications.


On the promise of the same shape of inputs and outputs of a neural network module, the input vector (with a 2-dimension shape) can be viewed as a group of particles in high-dimension space, and the propagation of this module can be viewed as movement of the group. As for global attention mechanism, the attention coefficients effect on the outputs, thus can be viewed as the interaction force between all pairs of particles, and help form the particle group into a system.
Therefore, algorithms on particle systems may be able to transferred to attention mechanism design and have a good influence on the graph transformer models with more theoretical explanation.


In this paper, we make contributions to both the theoretical research and efficient models of Transformer field. 
Specifically, 
\begin{itemize}
\item  We propose a theoretical model to characterize the propagation dynamics of the self-attention mechanism, offering a potential new tool for future research on Transformer theory and lay the foundation for random batch attention. 
\item We introduce Random Batch Attention (RBA), a novel self-attention mechanism derived from the Random Batch Methods, which reduces the complexity of standard self-attention while providing a mathematical proof to guarantee its expressivity. 
\item We apply RBA to GNNs to enhance memory efficiency, and experiments on large-scale graph datasets confirm that RBA does not compromise accuracy.
\end{itemize}

In Section 2, we present the necessary background knowledge, particularly the mathematical foundations of our research. Section 3 introduces our theoretical model of self-attention propagation. Section 4, as the core of this paper, details the proposed RBA algorithm, followed in Section 5 by a mathematical proof of its expressivity. Section 6 reports our experimental results, and Section 7 concludes the paper. Additional details and discussions of future work are provided in the Appendix.

\section{Preliminaries}

In this section, we claim all the definitions and symbols, and introduce necessary formulations and theorems used in following parts.

\subsection{Self-attention mechanism in Transformer}

Throughout the article, we regard the tokens with length $N$ and embedding dimension $d$ as $N$ particles in $\mathbb{R}^d$, and use ordinate $\left \{ X^i \right \} _{i=1}^N$ to represent it. Here, to be consistent with \cite{vaswani2017attention}, we represent the tokens as a row vector and denote $X=\left ( {X^1}^T {X^2}^T \cdots {X^N}^T \right )^T \in \mathbb{R}^{N \times d}$ as the token in matrix form.\footnote{In the Section \labelcref{convergence}, for convenience of derivation, we transpose tokens back into column vectors. }

To describe a dense neural network layer, we denote $W \in \mathbb{R}^{d \times d'}$ as the weight matrix and $\sigma(\cdot)$ as the activation function. So if  $X$ is taken to be the input, the corresponding output is $X'=\sigma(XW) \in \mathbb{R}^{d \times d'}$. 

As for Transformer encoder layer, to be detailed, 
the (multi-head) self-attention module can be represented as the equation
\begin{align*}
    & A^h(X)={\rm Softmax}\left(\dfrac{XW_Q^h(XW_K^h)^T}{\sqrt{d}}\right)XW_V^h, \\
    & X'=\sum_{h=1}^H A^h(X)W_O^h,
\end{align*}
where $X \in \mathbb{R}^{N \times d}$, weight matrices $W_Q^h, W_K^h, W_V^h \in \mathbb{R}^{d \times d'}$, $W_O^h \in \mathbb{R}^{d' \times d}$, $H$ is the number of heads, $d'$ is the dimension of each head, and

\begin{equation*}
\operatorname{Softmax}\bigl(Y_{N \times d}\bigr)=\left(\frac{e^{Y_{i j}}}{\sum_{k=1}^{d} e^{Y_{i k}}}\right)_{N \times d}.
\end{equation*}
For the convenience of theoretical research, without loss of generality, we set $H=1$, so the self-attention module can be written as 
\begin{equation}\label{eq3}
    X'={\rm Softmax}\left(\dfrac{XW_Q(XW_K)^T}{\sqrt{d}}\right)XW_VW_O,
\end{equation}
where $X \in \mathbb{R}^{N \times d}$, weight matrices $W_Q, W_K, W_V \in \mathbb{R}^{d \times d'}$, $W_O \in \mathbb{R}^{d' \times d}$. So the output $X'$ has the same shape with $X$. 
After the self-attention calculation, the output $X'$ added with the input $X$ is propagated to the feedforward layer. The self-attention calculation, skip connection and feedforward operation together construct the self-attention mechanism, which is one of the most important part of Transformer.


\subsection{Stochastic differential equations}

Stochastic differential equations (SDEs) describe dynamical systems influenced by random noise. Consider an SDE of the form:

\begin{equation} \label{itoSDE}
\mathrm{d}X_t = \mu_t(X_t)\mathrm{d}t + \sigma_t(X_t)\mathrm{d}W_t,
\end{equation}
where $X_t \in \mathbb{R}^n$ is the state process, $\mu: [0,T]\times\mathbb{R}^n \to \mathbb{R}^n$ is the drift term, $\sigma: [0,T]\times\mathbb{R}^n \to \mathbb{R}^{n\times m}$ is the diffusion term, and $W_t$ is an $m$-dimensional Wiener process. Under standard Lipschitz and linear growth conditions on $\mu$ and $\sigma$, this SDE admits a unique strong solution.

\subsubsection{Itô's formula.}

The Itô's formula plays a fundamental role in stochastic calculus.
In classical calculus, for a quadratic continuous differentiable function and a differentiable curve, we have a chain rule:
\begin{equation*}
    \mathrm{d}f(x(t))=f'(x(t))\mathrm{d}x(t).
\end{equation*}
However, in random analysis, when $x(t)$ is replaced by a Wiener provess $W_t$ or, more generally, an It process $X_t$, the classical chain rule fails. This is because Brownian motion paths, while continuous, are nowhere differentiable and have infinite variation. Kiyosi Itô's key contribution was identifying the essential difference and formulating Itô's formula, widely regarded as the cornerstone of stochastic calculus. 

We first consider \eqref{itoSDE} in the one-dimensional situation. Let $f(x,t)$ be a sufficiently smooth real-valued function.  Itô's formula gives the result
\begin{equation*}\label{ito1}
   \mathrm{d} f\left(t, X_{t}\right)=\left[\frac{\partial f}{\partial t}\left(t, X_{t}\right)+\mu_{t} \frac{\partial f}{\partial x}\left(t, X_{t}\right)+\frac{1}{2} \sigma_{t}^{2} \frac{\partial^{2} f}{\partial x^{2}}\left(t, X_{t}\right)\right] \mathrm{d} t+\sigma_{t} \frac{\partial f}{\partial x}\left(t, X_{t}\right) \mathrm{d} W_{t}.
\end{equation*}

And as a special case, if $f$ does not explicitly depend on $t$, the formula simplifies to: 
\begin{equation*} \label{ito2}
    \mathrm{d}f(X_t) = \left[ \mu_t f'(X_t) + \frac{1}{2} \sigma_t^2 f''(X_t) \right] \mathrm{d}t + \sigma_t f'(X_t) \mathrm{d}W_t. 
\end{equation*}

Compared to the classical chain rule, the formula includes an extra term, $\frac{1}{2} \sigma_t^2 \frac{\partial^2 f}{\partial x^2} \mathrm{d}t$, known as the \textbf{Itô correction term}. This arises from the non-zero quadratic variation of Brownian motion , that is $(\mathrm{d}W_t)^2 = \mathrm{d}t$.

Itô's formula naturally extends to functions of multiple stochastic variables. Consider \eqref{itoSDE} with dimension-$n$, or equivalently consider $m$ independent Brownian motions $W_t^1, \cdots, W_t^m$, and $n$ Itô processes
\begin{equation*}
    \mathrm{d}X_t^i = \mu_t^i \mathrm{d}t + \sum_{j=1}^m \sigma_t^{ij} \mathrm{d}W_t^j, \quad i = 1, \dots, n,
\end{equation*}

Itô's formula provides
\begin{equation*}\label{ito3}
    \begin{aligned}\mathrm{d} f\left(t, X_{t}\right)= & {\left[\frac{\partial f}{\partial t}\left(t, X_{t}\right)+\sum_{i=1}^{n} \mu_{t}^{i} \frac{\partial f}{\partial x_{i}}\left(t, X_{t}\right)+\frac{1}{2} \sum_{i=1}^{n} \sum_{j=1}^{n}\left(\sigma_t \sigma^{\top}_t \right)^{i j} \frac{\partial^{2} f}{\partial x_{i} \partial x_{j}}\left(t, X_{t}\right)\right] \mathrm{d} t } \\& +\sum_{i=1}^{n} \sum_{j=1}^{m} \frac{\partial f}{\partial x_{i}}\left(t, X_{t}\right) \sigma_{t}^{i j} \mathrm{d} W_{t}^{j}.\end{aligned}
\end{equation*}
or equivalently in vector form
\begin{equation*}
\mathrm{d}f(t,X_t) = \left(\frac{\partial f}{\partial t} + \mu_t\cdot\nabla_x f + \frac{1}{2}\text{tr}[\sigma_t\sigma^\top_t\nabla_x^2 f]\right)\mathrm{d}t + \nabla_x f\cdot\sigma_t \mathrm{d}W_t.
\end{equation*}

\subsubsection{Introduction to other results}
Several key results are essential for analyzing SDEs:

Grönwall's Inequality \citep{gronwall1919note}:
If $u(t) \leq \alpha(t) + \int_0^t \beta(s)u(s)ds$ with $\alpha$ non-decreasing and $\beta \geq 0$, then
\begin{equation}
u(t) \leq \alpha(t)\exp\left(\int_0^t \beta(s)ds\right).
\end{equation}

Fokker-Planck Equation \citep{fokker1914mittlere,planck1917satz}:
The transition density $p(t,x)$ of $X_t$ satisfies:
\begin{equation}
\frac{\partial p}{\partial t} = -\sum_{i=1}^n \frac{\partial (\mu_i p)}{\partial x_i} + \frac{1}{2}\sum_{i,j=1}^n \frac{\partial^2 ([\sigma\sigma^\top]_{ij}p)}{\partial x_i\partial x_j}.
\end{equation}

Burkholder-Davis-Gundy Inequality \citep{davis2011integral}:
For $p > 0$, $\exists C_p$ such that
\begin{equation}
\mathbb{E}\left[\sup_{0\leq s\leq t}\left|\int_0^s \sigma_r\mathrm{d}W_r\right|^p\right] \leq C_p\mathbb{E}\left[\left(\int_0^t \|\sigma_s\|^2ds\right)^{p/2}\right].
\end{equation}

These results provide the foundation for studying existence, uniqueness, and regularity properties of SDE solutions.

\subsection{Interacting particle systems theory}

Particle systems are complex dynamical systems composed of large numbers of interacting individuals, with wide applications in statistical physics, mathematical biology, and engineering sciences. Consider a system of $N$ particles, where the state of each particle is described by its position $X^i=X^i(t)\in \mathbb{R}^d$ and velocity $V^i=V^i(t) \in \mathbb{R}^d$ . 
$(i=1,...,N)$Such systems can typically be modeled by the following coupled system of stochastic differential equations:

\begin{equation}
\begin{cases}
\mathrm{d}X^i = V^i \mathrm{d}t, \\
\mathrm{d}V^i = \left[ -\gamma V^i + \frac{1}{N}\sum_{j=1}^N F(X^i-X^j) \right]\mathrm{d}t + \sigma \mathrm{d}W^i,
\end{cases}
\end{equation}

where $\gamma>0$ is the friction coefficient, $F:\mathbb{R}^d\to\mathbb{R}^d$ represents interparticle interaction forces, $\sigma>0$ is the noise intensity, and $\{W^i\}_{i=1}^N$ are independent d-dimensional standard Brownian motions. This model exhibits the following key features:

1. Mean-field limit: As the number of particles $N\to\infty$, the evolution of the empirical measure $\mu^N = \frac{1}{N}\sum_{i=1}^N \delta_{(X^i,V^i)}$ is described by the following nonlinear Fokker-Planck equation \cite{}:
\begin{equation}
\partial_t f + V\cdot\nabla_{X} f + \nabla_{V}\cdot\left[ (F*f - \gamma V)f \right] = \frac{\sigma^2}{2}\Delta_{V} f.
\end{equation}
where $f(t,X,V)$ is the particle density function in phase space, and $F*f = \int F(X-Y)f(t,Y,V)dYdV$ represents the mean-field interaction.

2. Propagation of chaos: If the initial conditions satisfy the chaoticity property (i.e., $\{(X^i(0),V^i(0))\}$ are independent and identically distributed), then for any finite time $T>0$, the particle states remain asymptotically independent as $N\to\infty$. This property can be rigorously proved using Grönwall's inequality

3. Energy balance: The total system energy $\mathcal{E} = \frac{1}{N}\sum_{i=1}^N |V^i|^2$ satisfies the following stochastic differential relation:
\begin{equation}
d\mathcal{E} = \left[ -2\gamma \mathcal{E}_ + \sigma^2 d \right] \mathrm{d}t + \frac{2\sigma}{N}\sum_{i=1}^N V^i \cdot \mathrm{d}W^i.
\end{equation}

The theoretical framework provides the mathematical foundation for studying emergent phenomena from microscopic stochastic dynamics to macroscopic deterministic behavior, with important applications in active matter, plasma physics, and other fields.

In this paper, we consider a specific situation. To describe an particle system, let $\left \{ X^i \right \}_{i=1}^N$ be the coordinate sequence of $N$ particles $\left \{ i \right \}_{i=1}^N$ with dimension $d$ of the system and $b(\cdot) \in \mathbf{C}(\mathbb{R}^d;\mathbb{R}^d)$ be the external force field. For the interaction of any pairs of particles $i$ and $j$, let $K_{ij} \in \mathbf{C}(\mathbb{R}^d \times \mathbb{R}^d;\mathbb{R}^d)$ be the interacting force with weight $m_{j} \ge 0 ,j=1,2,\cdots,N$. Denote $\left \{ W^i \right \} _{i=1}^N$ as a sequence of given independent $d$ dimensional Wiener processes (the standard Brownian motions),
  we can describe this particle system by the following stochastic differential equation

\begin{equation}\label{system}
    \mathrm{d} X^{i}=b\left(X^{i}\right) \mathrm{d} t+\frac{1}{N-1} \sum_{j: j \neq i} m_{j} K_{i j}\left(X^{i}, X^{j}\right) d t+\sigma \mathrm{d} W^{i} \quad i=1,2, \ldots, N.
\end{equation}
We briefly introduce the random empirical probability measure corresponding to \eqref{system}, by the definition
\begin{equation*}
    \mu^{N}(t):=\frac{1}{N} \sum_{j=1}^{N} \omega_{j} \delta\left(x-X^{j}(t)\right),
\end{equation*}
where 
\begin{equation*}
    \omega_{j}=\frac{N m_{j}}{\sum_{j} m_{j}},
\end{equation*}
and $\delta(x)$ is the Dirac measure that for any set $A$, $\delta(x)(A)=\left\{\begin{matrix}
1 , \ \ if \ x \in A
 \\
0 , \ \ if \ x \notin A.
\end{matrix}\right.$

Writing this measure as $\mu^{N}(t):=\dfrac{\sum_{j=1}^{N}m_j\delta\left(x-X^{j}(t)\right)}{\sum_{j=1}^{N}m_j}$, obviously,
\begin{equation*}
\mu^N(t)(A)=\dfrac{\sum_{j:x-\delta_j(t) \in A}m_j}{\sum_{j=1}^N m_j} 
\end{equation*}

The random empirical probability measure Established a mapping from discrete point sets to the real number field, equivalent to establishing a function about sequences. In \cite{min2022Transformer}, the author has done a research on the measure to measure flow map of the empirical probability measure to describe the Transformer propagation from a new point of view.

For convenience, a single symbol may represent different constant values throughout the paper.

\subsection{Taylor expansion of vector valued functions}

For convenience of the following theoretical study, we introduce an important tool, Taylor expansion of vector valued functions. While, we just consider the case of the variates and function values with the same dimension, because all the vectors are in the fixed vector space.

\subsubsection{For univariate functions}
Recall the Taylor expansion of scalar valued univariate functions.
The Taylor expansion is
\begin{align*}
f(x) &= \sum_{k=0}^{\infty} \frac{f^{(k)}(x_0)}{k!}(x-x_0)^k \\
     &= f(x_0) + f'(x_0)(x-x_0) + \frac{f''(x_0)}{2!}(x-x_0)^2 + \frac{f'''(x_0)}{3!}(x-x_0)^3 + \cdots.
\end{align*}
For many problems, we just need to expand to the second order, as the following
\begin{align*}
f(x) &= f(x_0) + f'(x_0)(x-x_0) + \frac{f''(x_0)}{2!}(x-x_0)^2 + o((x-x_0)^2) \\
     &= f(x_0) + f'(x_0)(x-x_0) + \frac{1}{2}f''(x_0)(x-x_0)^2 + o((x-x_0)^2) \\
     &= f(x_0) + f'(x_0)(x-x_0) + \frac{1}{2}M_2(x-x_0)^2,
\end{align*}
where $M_2$ is a coefficient related to $|f''(x_0)|$.

Similarly, let $\mathbf{x}\in \mathbb{R}^d$ and $f(\mathbf{x})\in \mathbb{R}^d$ also. We perform the expansion for each component of $f(\mathbf{x})$
\begin{align*}
f_i(\mathbf{x}) &= f_i(\mathbf{x}_0) + \nabla f_i(\mathbf{x}_0)^\top (\mathbf{x} - \mathbf{x}_0) + \frac{1}{2} (\mathbf{x} - \mathbf{x}_0)^\top \mathbf{H}_{f_i}(\mathbf{x}_0)(\mathbf{x} - \mathbf{x}_0) + o(\|\mathbf{x} - \mathbf{x}_0\|^2) \\
&= f_i(\mathbf{x}_0) + \sum_{j=1}^n \frac{\partial f_i}{\partial x_j}(\mathbf{x}_0)(x_j - x_{0j}) \\
&\quad + \frac{1}{2} \sum_{j=1}^n \sum_{k=1}^n \frac{\partial^2 f_i}{\partial x_j \partial x_k}(\mathbf{x}_0)(x_j - x_{0j})(x_k - x_{0k}) + o(\|\mathbf{x} - \mathbf{x}_0\|^2).
\end{align*}
Then the Taylor expansion to the second order is 
\begin{equation} \label{taylor}
\mathbf{f}(\mathbf{x}) = \mathbf{f}(\mathbf{x}_0) + \mathbf{J}_f(\mathbf{x}_0)(\mathbf{x} - \mathbf{x}_0) + \frac{1}{2} \begin{bmatrix}
(\mathbf{x} - \mathbf{x}_0)^\top \mathbf{H}_{f_1}(\mathbf{x}_0)(\mathbf{x} - \mathbf{x}_0) \\
(\mathbf{x} - \mathbf{x}_0)^\top \mathbf{H}_{f_2}(\mathbf{x}_0)(\mathbf{x} - \mathbf{x}_0) \\
\vdots \\
(\mathbf{x} - \mathbf{x}_0)^\top \mathbf{H}_{f_n}(\mathbf{x}_0)(\mathbf{x} - \mathbf{x}_0)
\end{bmatrix} + o(\|\mathbf{x} - \mathbf{x}_0\|^2) .
\end{equation}

For the sake of simplicity in form, we utilize the outer product and multiplication on high-dimensional matrices.

\paragraph{Outer product of vectors.}
Let $a \in \mathbb{R}^m,b \in \mathbb{R}^n$, the outer product of $a$ and $b$ is defined as 
\begin{equation*}
    a \otimes b := ab^T=(a_ib_j)_{m \times n}.
\end{equation*}

\paragraph{Multiplication of high-dimensional matrices.}
Let $A \in \mathbb{R}^{m \times n \times l}$ and $B \in \mathbb{R}^{n \times l}$, define the multiplication of $A$ and $B$ as
\begin{equation*}
    AB=\left ( \sum_{j=1}^n \sum_{k=1}^l A_{1jk}B_{jk}        \cdots     \sum_{j=1}^n \sum_{k=1}^l A_{mjk}B_{jk}  \right ) ^T.
\end{equation*}
The result is a m-dimension vector.
As a special example, let $C \in \mathbb{R}^{m \times n \times n}$ and $u,v \in \mathbb{R}^n$, we derive 
\begin{equation*}
    C(u \otimes v)_i=\sum_{j=1}^n\sum_{k=1}^nC_{ijk}u_jv_k=\sum_{k=1}^nv_k(\sum_{j=1}^nC_{ijk}u_j)=u^TC_iv.
\end{equation*}
So
\begin{equation*}
    C(u \otimes v)=(u^T C_1 v \cdots u^TC_mv)^T=u^TCv.
\end{equation*}
Therefore, if we denote $\mathbf{H}_f(\mathbf{x}_0)=(\mathbf{H}_{f_1}(\mathbf{x}_0) \cdots \mathbf{H}_{f_n}(\mathbf{x}_0))^T \in \mathbb{R}^{n\times n\times n}$, the Taylor expansion \eqref{taylor} can be written as
\begin{align*}
    \mathbf{f}(\mathbf{x}) &= \mathbf{f}(\mathbf{x}_0) + \mathbf{J}_f(\mathbf{x}_0)(\mathbf{x} - \mathbf{x}_0) + \frac{1}{2} (\mathbf{x} - \mathbf{x}_0)^T\mathbf{H}_f(\mathbf{x}_0)(\mathbf{x} - \mathbf{x}_0) + o(\|\mathbf{x} - \mathbf{x}_0\|^2) \\
    &=\mathbf{f}(\mathbf{x}_0) + \mathbf{J}_f(\mathbf{x}_0)(\mathbf{x} - \mathbf{x}_0) + \frac{1}{2} \mathbf{H}_f(\mathbf{x}_0)[(\mathbf{x} - \mathbf{x}_0)\otimes (\mathbf{x} - \mathbf{x}_0)] + o(\|\mathbf{x} - \mathbf{x}_0\|^2).
\end{align*}
We claim that there is a bound $M_2$ related to $\left \{ ||\mathbf{H}_{f_i}(\mathbf{x}_0)|| \right \}_{i=1}^n$ , one can obtain
\begin{equation*}
    \mathbf{f}(\mathbf{x}) = \mathbf{f}(\mathbf{x}_0) + \mathbf{J}_f(\mathbf{x}_0)(\mathbf{x} - \mathbf{x}_0) + \frac{1}{2} M_2(\mathbf{x} - \mathbf{x}_0)^T(\mathbf{x} - \mathbf{x}_0) \mathbf{1}_n
\end{equation*}
where $\mathbf{1}_n=(1 \cdots  1)^T \in \mathbb{R}^n$.

\subsubsection{For multivariate functions}

Without loss of generality, we just consider the case of two variates. Let $\mathbf{x},\mathbf{y} \in \mathbb{R}^n$ and $f(\mathbf{x},\mathbf{y})\in \mathbb{R}^d$ also, the Taylor expansion at point $(\mathbf{x}_0,\mathbf{y}_0)$ is
\begin{align*}
    \mathbf{f}(\mathbf{x},\mathbf{y}) = \ &\mathbf{f}(\mathbf{x}_0,\mathbf{y}_0) + \mathbf{J}_\mathbf{x}(\mathbf{x}_0,\mathbf{y}_0)(\mathbf{x} - \mathbf{x}_0) + \mathbf{J}_\mathbf{y}(\mathbf{x}_0,\mathbf{y}_0)(\mathbf{y} - \mathbf{y}_0)  \\ & + \frac{1}{2} \mathbf{H}_{xx}(\mathbf{x}_0,\mathbf{y}_0)[(\mathbf{x} - \mathbf{x}_0)\otimes (\mathbf{x} - \mathbf{x}_0)] + \mathbf{H}_{xy}(\mathbf{x}_0,\mathbf{y}_0)[(\mathbf{x} - \mathbf{x}_0)\otimes (\mathbf{y} - \mathbf{y}_0)] \\ & + \frac{1}{2} \mathbf{H}_{yy}(\mathbf{x}_0,\mathbf{y}_0)[(\mathbf{y} - \mathbf{y}_0)\otimes (\mathbf{y} - \mathbf{y}_0)]+o(\sqrt{||\mathbf{x}-\mathbf{x}_0||^2+||\mathbf{y}-\mathbf{y}_0||^2})\mathbf{1}_n.
\end{align*}
Then, there will be a coefficient related to ${\mathbf{H}_{xx}
(\mathbf{x}_0,\mathbf{y}_0)},\mathbf{H}_{xy}
(\mathbf{x}_0,\mathbf{y}_0),\mathbf{H}_{yy}
(\mathbf{x}_0,\mathbf{y}_0)$, satisfying
\begin{align*}
    \mathbf{f}(\mathbf{x},\mathbf{y}) = \ &\mathbf{f}(\mathbf{x}_0,\mathbf{y}_0) + \mathbf{J}_\mathbf{x}(\mathbf{x}_0,\mathbf{y}_0)(\mathbf{x} - \mathbf{x}_0) + \mathbf{J}_\mathbf{y}(\mathbf{x}_0,\mathbf{y}_0)(\mathbf{y} - \mathbf{y}_0)  \\ & + \frac{1}{2} M_2 \Big( (\mathbf{x} - \mathbf{x}_0)\otimes (\mathbf{x} - \mathbf{x}_0) + 2 (\mathbf{x} - \mathbf{x}_0)\otimes (\mathbf{y} - \mathbf{y}_0) + (\mathbf{y} - \mathbf{y}_0)\otimes (\mathbf{y} - \mathbf{y}_0) \Big) \\ 
    = \ &\mathbf{f}(\mathbf{x}_0,\mathbf{y}_0) + \mathbf{J}_\mathbf{x}(\mathbf{x}_0,\mathbf{y}_0)(\mathbf{x} - \mathbf{x}_0) + \mathbf{J}_\mathbf{y}(\mathbf{x}_0,\mathbf{y}_0)(\mathbf{y} - \mathbf{y}_0)  \\ & + \frac{1}{2} M_2 \Big( (\mathbf{x} - \mathbf{x}_0 + \mathbf{y} - \mathbf{y}_0)\otimes (\mathbf{x} - \mathbf{x}_0+\mathbf{y} - \mathbf{y}_0) \Big) .
\end{align*}

\section{Modeling for Attention and graph transformer by Particle system}
\label{sec:modeling}

For an Transformer encoder layer, the main part is the self-attention module, with the skip connection and layer normalization operation. In the following, we ignore the feed-forward part, just take a Transformer encoder layer as a self-attention layer.

We consider all the weight matrices were trained sufficiently and be constant. Notice that $\sqrt{d}$ is a constant, so we can substitute ${W_Q}$ for $\dfrac{W_Q}{\sqrt{d}}$. \eqref{eq3} can be written as 
\begin{equation}\label{eq4}
    X'={\rm Softmax}\left({XW_Q(XW_K)^T}\right)XW_VW_O.
\end{equation}
Then, we combine these weight matrices. Let $W=W_QW_K \in \mathbb{R}^{d \times d}, \hat{W}=W_VW_O \in \mathbb{R}^{d \times d}$, \eqref{eq4} can be written as 
\begin{equation}\label{eq5}
    X'={\rm Softmax}({XWX^T})X\hat{W}
\end{equation}

Now we derive \eqref{eq5}. Let $X^i \in \mathbb{R}^{1 \times d}$ be the embedded vector of the $i$-th token. So $X=({X^1}^T,\cdots,{X^n}^T)^T$.  
\begin{align*}
    X'
    & = \dfrac{\mathrm{exp}(X^iW{X^k}^T)}{\sum_{j=1}^n \mathrm{exp}(X^iW{X^j}^T)}_{n \times n} X \hat{W} \\
    & = \begin{pmatrix}
\sum_{k=1}^{N} \dfrac{\mathrm{exp}(X^iW{X^k}^T)}{\sum_{j=1}^N \mathrm{exp}(X^iW{X^j}^T)}{X^k}^T
 \\
\vdots 
 \\
\sum_{k=1}^{N} \dfrac{\exp(X^iW{X^k}^T)}{\sum_{j=1}^N \exp(X^iW{X^j}^T)}{X^k}^T
\end{pmatrix} \hat{W} \\
    & = \begin{pmatrix}
\sum_{k=1}^{N} \dfrac{\exp(X^iW{X^k}^T)}{\sum_{j=1}^N \mathrm{exp}(X^1W{X^j}^T)}{X^k}^T\hat{W}
 \\
\vdots 
 \\
\sum_{k=1}^{N} \dfrac{\exp(X^iW{X^k}^T)}{\sum_{j=1}^N \exp(X^nW{X^j}^T)}{X^k}^T\hat{W}
\end{pmatrix}. 
\end{align*}

If we also let ${X'}^i$ be the embedding vector of the $i$-th token, it can be derived from the above equation that
\begin{equation}\label{eq6}
    {X'}^i=\sum_{k=1}^{N} \dfrac{\exp(X^iW{X^k}^T)}{\sum_{j=1}^n \exp(X^iW{X^j}^T)}{X^k}^T\hat{W} .
\end{equation}

Next, we add the skip connection into \eqref{eq6}, that is 
\begin{equation}\label{eq7}
    {X'}^i=X^i + \sum_{k=1}^{N} \dfrac{\exp(X^iW{X^k}^T)}{\sum_{j=1}^N \exp(X^iW{X^j}^T)}{X^k}^T\hat{W} .
\end{equation}
From \eqref{eq7}, we can get the continuous equation of a self-attention layer:
\begin{equation}\label{eq8}
    \frac{\mathrm{d}}{\mathrm{d} t} X^i=\sum_{k=1}^{N} \dfrac{\exp(X^iW{X^k}^T)}{\sum_{j=1}^N \exp(X^iW{X^j}^T)}{X^k}^T\hat{W} .
\end{equation}

Notice that $n \gg d$, we can express $\dfrac{\mathrm{d}}{\mathrm{d}t}X^i$ with $\left\{ X^i \right\}_{i=1}^N$, let 
\begin{equation}\label{eq9}
    \frac{\mathrm{d}}{\mathrm{d} t} X^i=\sum_{k=1}^{N} \dfrac{\exp(X^iW{X^k}^T)}{\sum_{j=1}^N \exp(X^iW{X^j}^T)}w_k{X^k}^T .
\end{equation}

We consider how to represent the layer normalization. According to \cite{geshkovski2023mathematical}, the layer normalization is scaling the vector in order to guarantee that the vector is always on the unit sphere.  Define 
$P_x(y)=y-\dfrac{\langle x,y\rangle}{\langle x,x\rangle } x$. It is easy to know if $x \in S^{d-1}$, there will be $P_x(y) \perp x$.

So we change \eqref{eq9} into
\begin{equation}\label{eq10}
    \frac{\mathrm{d}}{\mathrm{d} t} X^i=P_{X^i}\left(\sum_{k=1}^{N} \dfrac{\exp(X^iW{X^k}^T)}{\sum_{j=1}^N \exp(X^iW{X^k}^T)} w_k{X^k}^T\right) .
\end{equation}

According to \cite{geshkovski2023mathematical}, for theoretical research, we change \eqref{eq10} as a surrogate model:
\begin{equation}\label{eq12}
    \frac{\mathrm{d}}{\mathrm{d} t} X^i=P_{X^i}\left(\dfrac{1}{N-1} \sum_{k=1}^{N} \exp(X^iW{X^k}^T){X^k}^T\right) .
\end{equation}
We change the denominator of Softmax to be $N-1$. Because the vector $X^i$ is always on the unit sphere, this change will not lead to blowing up caused by the sum of the exponents on the numerator.

Hence, there exists $\{  w_i \}_{i=1}^N$ such that
\begin{align*}
    \dfrac{\rm d}{{\rm d}t}X^i
    & = \dfrac{1}{N-1} \sum_{k=1}^{N} \exp(X^iW{X^k}^T){w_kX^k }\\
    & \hspace{1cm} - \left\langle \dfrac{1}{N-1} \sum_{k=1}^{N} \exp(X^iW{X^k}^T) w_k X^k , X^i \right\rangle \dfrac{X^i }{\langle X^i,X^i\rangle} \\
    & = \dfrac{1}{N-1} \sum_{k=1}^{N} \exp(X^iW{X^k}^T) w_k \left(X^k-\left\langle X^k,X^i\right\rangle \dfrac{X^i }{\langle X^i,X^i\rangle}\right ) \\
    & = \dfrac{1}{N-1} \sum_{k \ne i} \exp(X^iW{X^k}^T) w_k \left( X^k-\left\langle X^k,X^i\right\rangle X^i \right).
\end{align*}

Summing up the above, we establish a model of the propagation that
\begin{equation}\label{eq13}
    \dfrac{\rm d}{{\rm d} t}X^i = \dfrac{1}{N-1} \sum_{k \ne i} \exp(X^iW{X^k}^T) w_k \left( X^k- \left\langle X^k,X^i \right\rangle X^i\right).
\end{equation}

\section{Random batch methods for graph transformer}

For a particle system with $N$ particles with ordinates $\left \{ X^i \right \} _{i=1}^N$ \footnote{In this paper, we use superscripts to represent the order of particles (vectors) and subscripts to represent the components of particles (vectors) , such as $X^i$ represent the $i$-th particle and $X^i_j$ represent the $j$-th component of the $i$-th particle}, we describe its movement with a stochastic differential equation:
\begin{equation} \label{SDE}
    \mathrm{d}X^i=b(X^i) \mathrm{d}t+\frac{1}{N-1} \sum_{j:j \ne i} m_j K_{ij}(X^i,X^j)\mathrm{d}t+\sigma \mathrm{d}W^i,
\end{equation}
as mentioned in preliminaries.
We suppose that the accurate solution of  \eqref{SDE} can be solved directly. Thus, the complexity of solving this SDE is $O(N^2)$, which comes mainly from the calculation of ${K_{ij}}$.

\cite{RBM} and \cite{jin2021convergence} provide a fast algorithm\footnote{In fact, we use the version put forward in \cite{jin2021convergence} as its format is similar to RBA.}, named Random Batch Methods, for approximate solutions, which is shown with Algorithm~\ref{alg1}.

\begin{algorithm} 
\caption{Random Batch Methods} 
\label{alg1} 
\begin{algorithmic}
    \FOR{$k$ in $1: [T/\tau]$}
        \STATE Devide $\left \{ 1,2,\cdots,pn \right \}$ into $n$ batches randomly;
        \FOR{each batch $\mathcal{C}_q$}
        \STATE Update $\tilde{X}^i$ $(i \in \mathcal{C}_q)$ by solving the following SDE with $t \in  [t_{k-1}, t_k)$
    \begin{equation}\label{p-SDE}
            d\tilde{X}^i = b(\tilde{X}^i)\mathrm{d}t + \frac{1}{p-1} \sum_{j \in \mathcal{C}_q, j \neq i} m_j K_{ij} (\tilde{X}^i, \tilde{X}^j) \mathrm{d}t + \sigma \mathrm{d}W^i.
        \end{equation}
        \ENDFOR
    \ENDFOR
\end{algorithmic} 
\end{algorithm}

While, in  Algorithm~\ref{alg1}, this part of computation is cut down into a $O(pN)$ complexity. Except for the weighted sum of $K_{ij}$, every part in \eqref{SDE} has a no more than $O(N)$ complexity, which makes less contributions to the computation cost. Thus the whole complexity is greatly reduced.

Random Batch Methods provide a brand new perspective for us to improve attention mechanism, as attention computation is a special form of the weighted sum of interaction functions $K_{ij}$.

\subsection{RBA algorithm}

It is derivated in Section~\ref{sec:modeling} that self-attetnion propagation is reasonably viewed as a sort of SDE. Therefore, we can write down the Random batch version of Algorithm~\ref{alg1}.

\begin{algorithm} 
\caption{Random Batch Attention} 
\label{alg2} 
\begin{algorithmic}
    \STATE $X \gets \text{Concat}(X,{\rm padding})$.
    \STATE Divide $\left \{ 1,2,\cdots,\left \lceil \frac{N}{p}  \right \rceil \cdot p \right \}$ into $ p $ batches randomly.
    \FOR{each batch $X_q$}
        \item $X_q \gets \text{Attn}(X_q)$.
    \ENDFOR
    \STATE $X=\text{Concat}(X_1,X_2,\cdots,X_{\left \lceil \frac{N}{p}  \right \rceil})$.
    \STATE $X \gets X[:N]$.
\end{algorithmic} 
\end{algorithm}

One can obtain an intuitive understanding with the sketch map \labelcref{fig:enter-label}.

\begin{figure}
    \centering
    \includegraphics[width=1\linewidth]{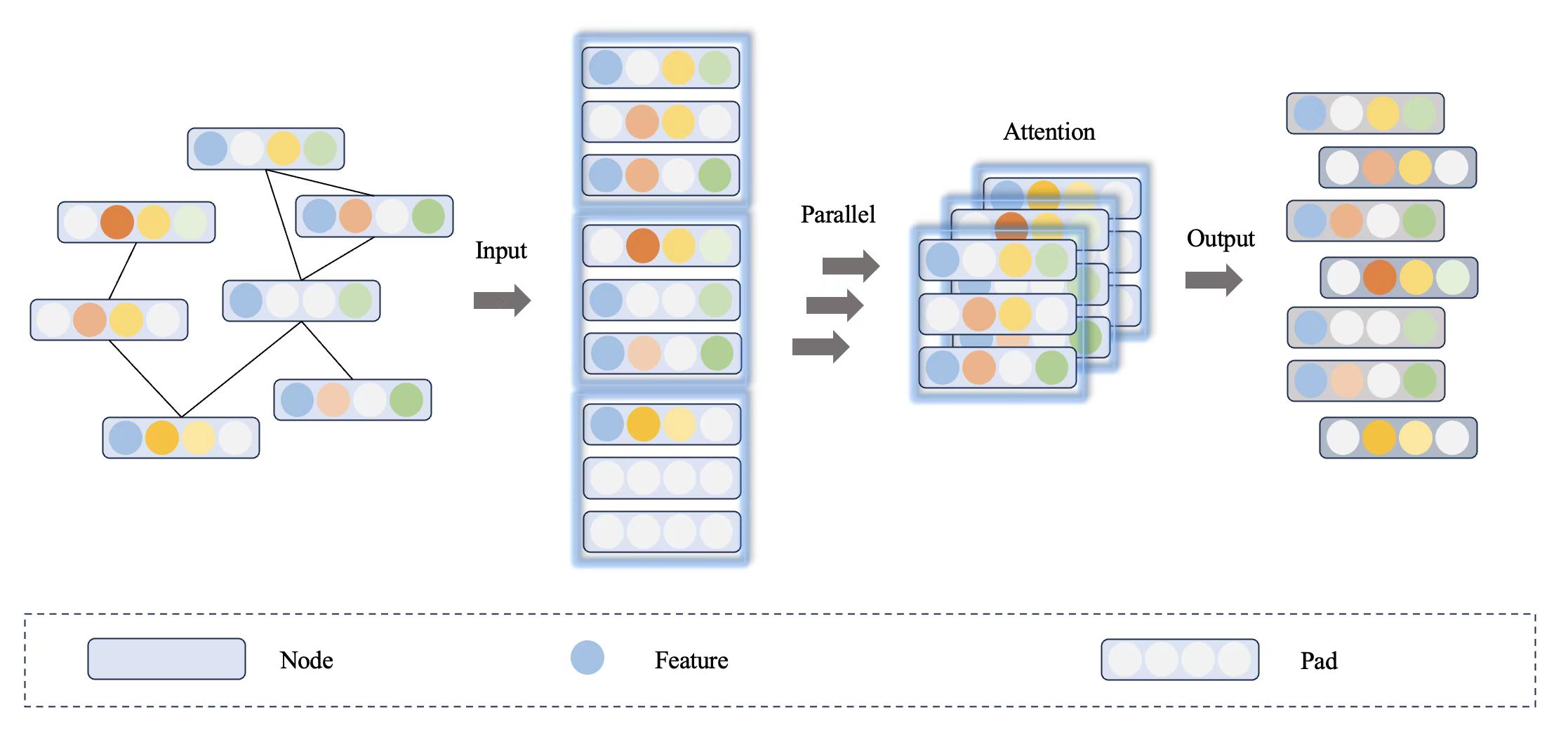}
    \caption{The Sketch Map of RBA}
    \label{fig:enter-label}
\end{figure}

The main difference between  Algorithms~\ref{alg1} and \ref{alg2} lies in the padding and truncating operation. Under the condition of long sequence input to graph transformer models, the effect of these two operations has small impact on the convergence of the algorithm.

Thus, we put forward the RBA algorithm with lower complexity for its practicality. Then, we go on to study the expressivity of RBA, which is the other important part of our algorithm.

\subsection{Convergence theorem of Random Batch Methods}

As shown in the following property, Random Batch Methods ares able to  gain much more accurate solution under some conditions.

We denote $\mathcal{C}_q^{(k)}(1 \le q \le n)$ as the batches at time $t_k$, then $\mathcal{C}^{(k)}:=\left \{ \mathcal{C}_1^{(k)} ,\cdots,\mathcal{C}_n^{(k)} \right \} $ denote the random division at time $t_k$.

Define the $\sigma$-algebra $\left \{ \mathcal{F}_{k} \right \} $ and $\left \{ \mathcal{G}_{k} \right \}$:
\begin{align}
    \mathcal{F}_{k-1}=\sigma(X_0^i,\mathcal{C}^{(j)};t \le t_{k-1},j \le k-1), \\
    \mathcal{G}_{k-1}=\sigma(X_0^i,\mathcal{C}^{(j)};t \le t_{k-1},j \le k-2).
\end{align}

It can be obtained that $\mathcal{F}_{k-1}=\sigma(\mathcal{G}_{k-1} \bigcup \sigma(\mathcal{C}_q^{(k-1)}))$, which contains the information of random division in time $[t_{k-1},t_k)$.

For the set $\mathcal{E}$ containing all the random variables that 
\begin{equation*}
    \mathcal{E}=\left \{ X_0^i,\mathcal{C}^{(j)};t \ge 0, 1 \le i \le N, j \ge 0 \right \},
\end{equation*}
according to the Kolmogorov extension theorem, there exists a probability space $(\Omega,\tilde{\mathcal{E}},\mathbb{P})$, such that $\mathcal{E}$ is on this space. We can perform mathematical expectation operations $\mathbb{E}(\cdot)$, integrating a random variable on $\Omega$ with respect to $\mathbb{P}$.

Define the norm $||\nu||=\sqrt{\mathbb{E}(|\nu|^2)}$ .

We define the error between the accurate solution of the SDE \eqref{SDE} $X^i(t)$ and its corresponding approximate solution $\tilde{X}^i(t)$ with Random Batch Methods by

\begin{equation}\label{eq:Jt}
    J(t) = \dfrac{1}{2N}\sum_{i=1}^N \mathbb{E}(|\tilde{X}^i(t)-X^i(t)|^2).
\end{equation}

\begin{lemma}\label{lem:Jt upper bd}
    For the particle system represented by the stochastic differential equation
\begin{equation}
    \mathrm{d}X^i=b(X^i) \mathrm{d}t+\frac{1}{N-1} \sum_{j:j \ne i} m_j K_{ij}(X^i,X^j)\mathrm{d}t+\sigma dB^i,
\end{equation}

if the function $b$ and $K$ satisfies the following:
\begin{enumerate}
    \item[(1)] $\dfrac{1}{N} \sum_{j=1}^N m_j=M$, $\max_{j} \left | m_j \right | \le A$ , $M$ and $A$ are constant.
    \item[(2)] $b$ is one-sided Lipschitz: $(z_1-z_2) \cdot (b(z_1)-b(z_2)) \le \beta \left | z_1-z_2 \right | ^2$.
    \item[(3)] $b$ and  $\nabla b$ have polynomial growth: $\left | b(z) \right |+\left | b(z) \right | \le C(1+|z|)^q$.
    \item[(4)] $K_{ij}$ and their derivatives up to second order (3 totally) are uniformly bounded in $1 \le i,j \le N$.
\end{enumerate}

then RBM has the following upper bound
\begin{equation}\label{}
    \sup_{t \le T}J(t) \le C(T) \left(\dfrac{\tau}{p-1}+\tau^2 \right).
\end{equation}
where $C(T)$ is a constant depend on $T$ and $\tau$ is the time step.
\end{lemma}


\section{The convergence theorem of RBA and proof}
\label{convergence}

In this section, we propose the following the convergence theorem for RBM self-attention mechanism  based on the above Lemma~\ref{lem:Jt upper bd}. 
To be mentioned, for convenience, we transpose all the token vectors into column vectors, which means $X^i \in \mathbb{R}^d$.

For the convenience of the following proof,  we denote 
\begin{equation*}
\label{defs}
\begin{aligned}
    Z^i&=\tilde{X}^i-X^i,  \\
F_i(x)&=\dfrac{1}{N-1} \sum_{j:j \ne i}K_{ij}(x^i,x^j), \\
\chi_i(x)&=\dfrac{1}{p-1}\sum_{j \in \mathcal{C}_i, j \ne i}K_{ij}(x_i,x_j)-\dfrac{1}{N-1}\sum_{j: j \ne i}K_{ij}(x_i,x_j),
\end{aligned}
\end{equation*}
where $x_i \in \mathbb{R}^d$ for all $i$ (thus $x \in \mathbb{R}^{Nd}$), $\mathcal{C}_i$ is the random batch containing $x_i$ in a random division, and $K_{ij}(X^i,X^j)=exp({X^i}^TW^T{X^j}) w_j \left( X^j-\left\langle X^j,X^i \right\rangle X^i  \right)$ described in the following \eqref{theorem1.1} and \eqref{theorem1.2} .


\begin{theorem}\label{theorem}
    For the process
\begin{equation}\label{theorem1.1}
    {\mathrm d}X^i = \dfrac{1}{N-1} \sum_{j:j \ne i} \exp({X^i}^TW^T{X^j}) w_j \left( X^j-\left\langle X^j,X^i \right\rangle X^i  \right) \mathrm{d}t
\end{equation}
with fixed parametres $W$ and $\{w_i\}_{i=1}^N$, in time $t \in [t_{k-1},t_k)$, where $t_k=k\tau$ and $k=1, \cdots, [T/\tau]$ ,
if we change it as 
\begin{equation}\label{theorem1.2}
    \mathrm{d}X^i = \dfrac{1}{p-1} \sum_{j \in \mathcal{C}_i,j \ne i} \exp({X^i}^TW^T{X^j}) w_j \left( X^j-\left\langle X^j,X^i \right\rangle X^i  \right) \mathrm{d}t,
\end{equation}
the solution $X^i$ and $\tilde{X}^i$ with time step $\tau$ has an error bound that
\begin{equation*}
     \sup_{t \le T}J(t) \le C(T) (\dfrac{\tau}{p-1}+\tau^2).
\end{equation*}
\end{theorem}

Before proving this theorem, we will first establish and prove five lemmas that will be used in the main proof.

\begin{lemma} \label{lemma2}
    For different $i,j,k$, it holds that
\begin{equation*}
\begin{aligned}
    \mathbb{E}I_{ij}&=\dfrac{p-1}{N-1},\\
    \mathbb{E}(I_{ij}I_{il})&=\dfrac{(p-1)(p-2)}{(N-1)(N-2)}.
\end{aligned}
\end{equation*}
where the expectation is taken with respect to the random division of batches.
\end{lemma}
   
\begin{proof}
We note that $I_{ij}=1$ if $i$ and $j$ are in the same batch, $I_{ij}=0$ if else. 

We begin with a set of $N=np$ distinct particles. The total number of ways to divide these particles into $n$ batches, with each batch containing $p$ particles, is given by the multinomial coefficient:
$M(n,p)=\frac{(pn)!}{(p!)^n n!}$.
Next, we consider the number of these arrangements in which two specific particles, say $i$ and $j$, are placed in the same batch.
To calculate this, we can follow a two-step process. First, we form a batch containing particles $i$ and $j$. This requires selecting an additional $p-2$ particles from the remaining $N-2$ particles. The number of ways to do this is $\binom{N-2}{p-2}$.
Second, once this first batch is formed, we are left with $N-p$ particles to distribute into the remaining $n-1$ batches, each of size $p$. The number of ways to perform this second step is equivalent to the total number of divisions for a smaller set, which is $M(n-1,p)$.
Therefore, the number of ways to divide the particles such that $i$ and $j$ are in the same batch is given by the product:
$\binom{N-2}{p-2} \cdot M(n-1,p)$. Thus, 
\begin{equation*}
    \mathbb{E}I_{ij}=\mathbb{P}(I_{ij}=1)=\dfrac{\binom{N-2}{p-2}M(n-1,p)}{M(n,p)}=\dfrac{p-1}{N-1}.
\end{equation*}

Similarly, for dividing particle $i,j$ and $i,k$ into the same batch, there are $\binom{N-3}{p-3}$ ways. Then there are $M(n-1,p)$ ways to divide the other particles into $n-1$ batches of size $p$. Then,
\begin{equation*}
    \mathbb{E}(I_{ij}I_{ik})=\mathbb{P}(I_{ij}=1,I_{ik}=1)=\dfrac{\binom{N-3}{p-3}M(n-1,p)}{M(n,p)}=\dfrac{(p-1)(p-2)}{(N-1)(N-2)}.
\end{equation*}
\end{proof}

\begin{lemma} \label{lemma3}
Let $x=(x_1,\cdots,x_N) \in \mathbb{R}^{N d}$, where $x_i$ is the coordinates of particle $i$. It holds that
\begin{equation*}
    \mathbb{E}{\chi_i}(x)=0.
\end{equation*}
and then
\begin{equation*}
    \mathrm{Var}(\chi_i(x))=\mathbb{E}|{\chi_i}(x)|^2=\left(\dfrac{1}{p-1}-\dfrac{1}{N-1}\right)\Lambda_i(x),
\end{equation*}
where $\chi_i(x)$ is given by the last definition in \eqref{defs}, and
\begin{equation*}
    \Lambda_i(x):=\dfrac{1}{N-2}\sum_{j:j \ne i}\bigl|K_{ij}(x_i,x_j)-\dfrac{1}{N-1}\sum_{k:k \ne i}K_{ik}(x_i,x_k)\bigr|^2.
\end{equation*}
\end{lemma}

\begin{proof}
Note that $\sum_{j: j \ne i}K_{ij}(x_i,x_j)$ is independent of the random division. Denote by
\begin{equation*}
    f_i(x):=\dfrac{1}{p-1}\sum_{j \in \mathcal{C}^i, j \ne i}K_{ij}(x_i,x_j).
\end{equation*}
Thus,
\begin{align*}
    \mathbb{E}f_i(x)
    & = \mathbb{E}\left\{\dfrac{1}{p-1}\sum_{j: j \ne i}K_{ij}(x_i,x_j) I_{ij}\right\}\\
    & = \dfrac{1}{p-1}\sum_{j: j \ne i}K_{ij}(x_i,x_j)\mathbb{E}I_{ij} \\
    & = \dfrac{1}{p-1}\sum_{j: j \ne i}K_{ij}(x_i,x_j)\dfrac{p-1}{N-1} \\
    & = \dfrac{1}{N-1}\sum_{j: j \ne i}K_{ij}(x_i,x_j).
\end{align*}
So 
\begin{equation*}
    \mathbb{E}{\chi_i}(x)=0.
\end{equation*}

\begin{align*}
    \mathbb{E}(f_i(x)^2) =
    & \dfrac{1}{(p-1)^2}\sum_{j: j \ne i}|K_{ij}(x_i,x_j)|^2 \mathbb{E}I_{ij} \\
    & + \dfrac{1}{(p-1)^2} \sum_{j,k:j \ne i,k \ne i,j \ne k}K_{ij}(x_i,x_j)K_{ik}(x_i,x_k) \mathbb{E}(I_{ij}I_{ik}) 
\end{align*}
Thus
\begin{align*}
    \mathrm{Var}(\chi_i(x))
    & = \mathrm{Var}(f_i(x)) \\
    & = \mathbb{E}(f_i(x)^2)-(\mathbb{E}f_i(x))^2 \\
    & = \dfrac{1}{(p-1)^2}\sum_{j: j \ne i}K_{ij}(x_i,x_j)^2 \mathbb{E}I_{ij}  \\
    & \ \ \ + \dfrac{1}{(p-1)^2} \sum_{j,k:j \ne i,k \ne i,j \ne k}K_{ij}(x_i,x_j)K_{ik}(x_i,x_k) \mathbb{E}(I_{ij}I_{ik}) \\
    & \ \ \ - (\dfrac{1}{N-1}\sum_{j: j \ne i}K_{ij}(x_i,x_j))^2 \\
    & = (\dfrac{1}{p-1}-\dfrac{1}{N-1}) ( \dfrac{1}{(N-1)} \sum_{j: j \ne i}K_{ij}(x_i,x_j)^2  \\
    & \ \ \ - \dfrac{1}{(N-1)(N-2)} \sum_{j,k:j \ne i,k \ne i,j \ne k}K_{ij}(x_i,x_j)K_{ik}(x_i,x_k) ) \\
    & = (\dfrac{1}{p-1}-\dfrac{1}{N-1}) \dfrac{1}{N-2} \sum_{j:j \ne i}\left|K_{ij}(x_i,x_j)-\dfrac{1}{N-1}\sum_{k:k \ne i}K_{ik}.(x_i,x_k)\right|^2.
\end{align*}
\end{proof}

\begin{lemma} \label{lemma4}
Denote $t_k=k\tau$.
Let $X^i$ and $\tilde{X}^i$ be solutions to \eqref{theorem1.1} and  \eqref{theorem1.2} respectively. We have
\begin{equation}\label{eq16}
    \sup_{t \le T} (\mathbb{E}|X^i(t)|^q+\mathbb{E}|\tilde{X}^i(t)|^q) \le C_q.
\end{equation}
Besides, for all $q \ge 2$ and $k >0$, it holds almost surely that
\begin{equation*}
    \sup_{t \in [t_{k-1},t_k)} \left | \mathbb{E}(|\tilde{X}^i(t)|^q|\mathcal{F}_{k-1}) \right | \le C_1 |\tilde{X}^i(t_{k-1})|^q + C_2.
\end{equation*}
Furthermore, it holds almost surely that 
\begin{equation}\label{eq17}
\begin{aligned}[c]
\left|\mathbb{E}(\tilde{X}(t)-\tilde{X}(t_{k-1})|\mathcal{F}_{k-1})\right| \le C(1+|\tilde{X}(t_{k-1})|^q) \tau,
\\
\left| \mathbb{E}(|\tilde{X}(t)-\tilde{X}(t_{k-1})|^2|\mathcal{F}_{k-1}) \right| \le  C(1+|\tilde{X}(t_{k-1})|^q) \tau.
\end{aligned}
\end{equation}
where $C_q,C_1,C_2,C$ are constants.
\end{lemma}


\begin{proof}
For our process equation 
\begin{equation*}
    \mathrm{d} X^i = \dfrac{1}{N-1} \sum_{j:j \ne i}K_{ij}(X^i,X^j) \mathrm{d}t,
\end{equation*}
we can obtain the derivative of the $q$-moment
\begin{equation*}
    \dfrac{\mathrm{d}}{\mathrm{d}t} \mathbb{E}(|X^i|^q) = \dfrac{q}{N-1} \mathbb{E} (|X^i|^{q-2}) (\sum_{j:j \ne i}K(X^i,X^j)\cdot X^i).
\end{equation*}
So 
\begin{equation*}
    \dfrac{\mathrm{d}}{\mathrm{d}t} \mathbb{E}(|X^i|^q) \le C\dfrac{q}{N-1} \mathbb{E}(|X^i|^{q-1}).
\end{equation*}
By Young's inequality,
\begin{equation*}
    \mathbb{E} (|X^i|^{q-1}) \le \dfrac{(q-1)\nu}{q} \mathbb{E}(|X^i|^q) + \dfrac{1}{q \nu^{q-1}}.
\end{equation*}
Choosing fixed $\nu$, there will be
\begin{equation*}
    \dfrac{\mathrm{d}}{\mathrm{d}t} \mathbb{E}(|X^i|^q) \le r_1 \dfrac{q}{N-1} \mathbb{E}(|X^i|^q) + r_2.
\end{equation*}
Hence, 
\begin{equation*}
    \sup_{t \in [t_k,t_{k-1})} \mathbb{E}(|X^i|^q) \le C_1 |X^i(t_{k-1})|^q + C_2.
\end{equation*}
And then, because $X^i$ is on the unit sphere, $\mathbb{E}(|X^i|^q)$ is uniformly bounded in time with any $q \ge 2$. 

For $\tilde{X}$,
\begin{equation*}
    \dfrac{\mathrm{d}}{\mathrm{d}t} \mathbb{E}(|\tilde{X}^i|^q|\mathcal{F}_{k-1}) = \dfrac{q}{N-1} \mathbb{E} (|\tilde{X}^i|^{q-2}|\mathcal{F}_{k-1}) (\sum_{j:j \ne i}K(\tilde{X}^i,\tilde{X}^j)\cdot \tilde{X}^i).
\end{equation*}
Similarly, we can obtain that
\begin{equation*}
    \dfrac{\mathrm{d}}{\mathrm{d}t} \mathbb{E}(|\tilde{X}^i|^q|\mathcal{F}_{k-1}) \le r_1 \dfrac{q}{N-1} \mathbb{E}(|\tilde{X}^i|^q | \mathcal{F}_{k-1}) + r_2.
\end{equation*}
So 
\begin{equation}\label{eq18}
    \sup_{t \in [t_k,t_{k-1})} \mathbb{E}(|\tilde{X}^i|^q | \mathcal{F}_{k-1}) \le C_1|\tilde{X}^i(t_{k-1})|^q+C_2.
\end{equation}
Taking expectation about the randomness in $\mathcal{F}_{k-1}$ on both sides of  \eqref{eq18}, we can get a similar control for $\tilde{X}^i$ as that for $X^i$:
\begin{equation*}
    \sup_{t \in [t_k,t_{k-1})} \mathbb{E}(|\tilde{X}^i|^q) \le C_1|\tilde{X}^i(t_{k-1})|^q+C_2.
\end{equation*}
And then $\mathbb{E}(|\tilde{X}^i|^q)$ is uniformly bounded in time with any $q \ge 2$.

Therefore, \eqref{eq16} follows. (\eqref{eq16} holds obviously when $q=0$ or $q=1$.)

Then we prove \eqref{eq17}.
Note that the process of $\tilde{X}^i$ is 
\begin{equation*}
    \dfrac{\mathrm{d}}{\mathrm{d}t}\tilde{X}^i=\dfrac{1}{p-1} \sum_{j \in \mathcal{C}_i,j \ne i}K_{ij}(\tilde{X}^i,\tilde{X}^j) .
\end{equation*}
Then
\begin{align*}
\mathbb{E}(\tilde{X}^i(t)-\tilde{X}^i(t_{k-1})|\mathcal{F}_{k-1}) 
& = \int_{t_{k-1}}^t \mathbb{E} \left (\dfrac{1}{p-1} \sum_{j \in \mathcal{C}_i,j \ne i}K_{ij}(\tilde{X}^i(s),\tilde{X}^j(s)) \Big| \mathcal{F}_{k-1} \right ) ds \\
& \le C\tau \le C(1+|\tilde{X}^i(t_{k-1})|)\tau.
\end{align*}
 
By It$\hat{\mathrm{o}}$ 's formula, 
\begin{align*}
& \dfrac{\mathrm{d}}{\mathrm{d}t}\mathbb{E}\left ( |\tilde{X}^i(t)-\tilde{X}^i(t_{k-1})|^2 \big\vert \mathcal{F}_{k-1} \right )  \\ =\
& 2 \mathbb{E}\left ( (\tilde{X}^i(t)-\tilde{X}^i(t_{k-1})) \cdot \dfrac{1}{p-1}\sum_{j \in \mathcal{C}_i,j \ne i}K_{ij}(\tilde{X}^i,\tilde{X}^j)\big\vert \mathcal{F}_{k-1} \right ) \\ =\
& 2 \mathbb{E}\left ( (\tilde{X}^i(t)-\tilde{X}^i(t_{k-1})) \cdot \dfrac{1}{p-1}\sum_{j \in \mathcal{C}_i,j \ne i}(K_{ij} (\tilde{X}^i(t),\tilde{X}^j(t))-K_{ij} (\tilde{X}^i(t_{k-1}),\tilde{X}^j(t_{k-1})))\big\vert \mathcal{F}_{k-1} \right ) \\ 
& + 2 \mathbb{E}\left ( (\tilde{X}^i(t)-\tilde{X}^i(t_{k-1})) \cdot \dfrac{1}{p-1}\sum_{j \in \mathcal{C}_i,j \ne i}K_{ij}(\tilde{X}^i(t_{k-1}),\tilde{X}^j(t_{k-1}))\big\vert \mathcal{F}_{k-1} \right ) \\ \le \
& 2\mathbb{E}\left ( |\tilde{X}^i(t)-\tilde{X}^i(t_{k-1})|^2 \big\vert \mathcal{F}_{k-1} \right ) + C \mathbb{E}\left ( |\tilde{X}^i(t)-\tilde{X}^i(t_{k-1})| \big\vert \mathcal{F}_{k-1} \right ) \\ \le \
& 2\mathbb{E}\left ( |\tilde{X}^i(t)-\tilde{X}^i(t_{k-1})|^2 \big\vert \mathcal{F}_{k-1} \right ) + C (1+|\tilde{X}^i(t_{k-1})|^q),
\end{align*} 
where the second last inequality holds because $\left \{K_{ij} \right \}$ and their derivatives are uniformly bounded.
Hence, the second inequality of \eqref{eq17} follows.
\end{proof}

\begin{lemma} \label{lemma5}
    Fix $i \in {1,\cdots,N}$. Let $\mathcal{C}_i$ be the random batch of size $p$ that containing particle $i$ in the random division. Let $Y_j$ be $N$ random variables or random vectors that are independent of $\mathcal{C}_i$. Then, for every $p \ge 2$, it holds that
\begin{equation}\label{eq19}
    \left \| \dfrac{1}{p-1} \sum_{j \in \mathcal{C}_i, j \ne i} Y_j   \right \|^2 \le \dfrac{1}{N-1}\sum_{j:j \ne i} \left \| Y_j \right \|^2.
\end{equation}
\end{lemma}

\begin{proof}
\begin{align*}
\left \| \dfrac{1}{p-1} \sum_{j \in \mathcal{C}_i, j \ne i} Y_j   \right \|^2
& = \dfrac{1}{(p-1)^2} \mathbb{E}\left\{(\sum_{j:j \ne i}I_{ij}Y_j)^2\right\} \\
& = \dfrac{1}{(p-1)^2} \sum_{j,k:j \ne i,k \ne i}\mathbb{E}(I_{ij}I_{ik})\mathbb{E}(Y_jY_k) \\
& \le \dfrac{1}{(p-1)^2} (\dfrac{p-1}{N-1}  \sum_{j:j \ne i}||Y_j||^2 \\
& \ \ \ + \dfrac{(p-1)(p-2)}{(N-1)(N-2)} \sum_{j,k:j \ne i,k \ne i, j \ne k} ||Y_j|| ||Y_k||) \\
& \le \dfrac{p-2}{(p-1)(N-1)(N-2)} \sum_{j,k:j \ne i,k \ne i,j \ne k}(\frac{1}{2}||Y_j||^2+\frac{1}{2}||Y_k||^2) \\
& \ \ \ + \dfrac{1}{(p-1)(N-1)} \sum_{j:j \ne i}||Y_j||^2 \\
& = (\dfrac{p-2}{(p-1)(N-1)} + \dfrac{1}{(p-1)(N-1)}) \sum_{j:j \ne i}||Y_j||^2\\
& = \dfrac{1}{N-1} \sum_{j:j \ne i}||Y_j||^2.
\end{align*}
So \eqref{eq19} holds.
The second equality is because of the independence and the first inequality is the result of Cauchy inequality and \labelcref{lemma2}.
\end{proof}

Now we calculate for the process of $Z^i$.
Define 
\begin{equation*}
    \delta K_{ij}(t):=K_{ij}(\tilde{X}^i(t),\tilde{X}^j(t))-K_{ij}(X^i(t),X^j(t)).
\end{equation*}
There will be
\begin{align*}
    \dfrac{\mathrm{d}}{\mathrm{d}t}Z^i
    & = \dfrac{1}{p-1}\sum_{j \in \mathcal{C}_i,j \ne i}K_{ij}(\tilde{X}^i,\tilde{X}^j)-\dfrac{1}{N-1}\sum_{j:j \ne i}K_{ij}(X^i,X^j) \\
    & = \dfrac{1}{N-1}\sum_{j:j \ne i}\delta K_{ij}+\chi_i(\tilde{X}) \\
    & = \dfrac{1}{p-1}\sum_{j \in \mathcal{C}_i,j \ne i}\delta K_{ij}+\chi_i(X).
\end{align*}

\begin{lemma} \label{lemma6}
For $t \in [t_{k-1},t_k)$, it holds that
\begin{equation}\label{eq20}
    \| Z^i(t)-Z^i(t_{k-1})\| \le C\tau,
\end{equation}
and almost surely that
\begin{equation}\label{eq21}
    |Z^i(t)-Z^i(t_{k-1})| \le C\tau.
\end{equation}
\end{lemma}

\begin{proof}
Because $K_{ij}$ is bounded, $\dfrac{\mathrm{d}}{\mathrm{d}t}Z^i(t) \le C$, which gives \eqref{eq21}.
And \eqref{eq20} is a simple corollary.
\end{proof}

In the following, we prove the convergence property in \labelcref{theorem}.

\textbf{Proof of Theorem~\ref{theorem}}.
Firstly, we let $M_0, M_1$ be the upper bound of $K_{ij}$ and their derivative respectively.
\begin{equation}\label{prooftheorem1}
    \dfrac{\mathrm{d}J(t)}{\mathrm{d}t}=\dfrac{1}{N} \sum_{i=1}^N \mathbb{E}\left\{Z^i \cdot (F_i((\tilde{X})-F_i(X))+\chi_i(\tilde{X}))\right\}.
\end{equation}
Due to the boundedness of $K_{ij}$, we can control the first term of \eqref{prooftheorem1} 
\begin{align*}
    \dfrac{1}{N} \sum_{i=1}^N \mathbb{E}\left\{Z^i \cdot \left(F_i(\tilde{X})-F_i(X) \right)\right\}
    & = \dfrac{1}{N} \sum_{i=1}^N \mathbb{E}\left\{Z^i \cdot \sum_{j \ne i}K_{ij}(X^i,X^j)\right\}\\
    & \le \dfrac{1}{N} \sum_{i=1}^N\mathbb{E}(|Z^i|^2+|Z^i||Z^j|)M_1 \\
    & \le CJ(t).
\end{align*}
Now, we consider the second term, and rewrite it into two parts as 
\begin{align*}
    \dfrac{1}{N} \sum_{i=1}^N \mathbb{E}\left \{Z^i(t) \cdot \chi_i(X(t))\right \} 
    & = \dfrac{1}{N} \sum_{i=1}^N \mathbb{E} \left\{Z^i(t_{k-1}) \cdot \chi_i(X(t)) \right\} \\
    & \ \ \ + \dfrac{1}{N} \sum_{i=1}^N \mathbb{E} \left\{ \left(Z^i(t)-Z^i(t_{k-1})\right) \cdot \chi_i(X(t)) \right\} \\
    & =: I_1 + I_2.
\end{align*}
We now estimate $I_1$.
Since 
\begin{equation*}
    \mathbb{E}\left\{Z^i(t_{k-1}) \cdot \chi_i(\tilde{X}(t_{k-1}))\right\}=
    \mathbb{E}\left\{Z^i(t_{k-1}) \cdot \mathbb{E}(\chi_i(\tilde{X}(t_{k-1})|\mathcal{G}_{k-1})\right\} = \mathbb{E}0=0,
\end{equation*}
we can rewrite $I_1$ as
\begin{equation*}
    I_1=\dfrac{1}{N} \sum_{i=1}^N \mathbb{E} \left\{ Z^i(t_{k-1}) \cdot \bigl(\chi_i(\tilde{X}(t))-\chi_i(\tilde{X}(t_{k-1}))\bigr) \right\}.
\end{equation*}
Note that 
\begin{align*}
    & \mathbb{E} \left\{ Z^i(t_{k-1}) \cdot \left(\chi_i(\tilde{X}(t))-\chi_i(\tilde{X}(t_{k-1}))\right) \right\} \\
  &=  \mathbb{E} \left\{Z^i(t_{k-1}) \cdot \mathbb{E}\left(\chi_i(\tilde{X}(t))-\chi_i(\tilde{X}(t_{k-1}))\big\vert\mathcal{F}_{k-1}\right) \right\} \\
  &\le  C \|Z^i(t_{k-1})\| \left \| \mathbb{E} \left( \chi_i(\tilde{X}(t))-\chi_i(\tilde{X}(t_{k-1}))\big\vert \mathcal{F}_{k-1} \right) \right \| .
\end{align*}
For $t \in [t_{k-1},t_k)$, denote by
\begin{equation*}
\begin{aligned}
    \delta \tilde{K}_{ij}(s) &:= K_{ij}(\tilde{X}^i(s),\tilde{X}^j(s))-K_{ij}(\tilde{X}^i(t_{k-1}),\tilde{X}^j(t_{k-1})), \\
    \delta \tilde{X}_{j}(s) &:= \tilde{X}^j(s) - \tilde{X}^j(t_{k-1}).
\end{aligned}
\end{equation*}

Then there will be
\begin{equation*}
    \mathbb{E} \left( \chi_i(\tilde{X}(t))-\chi_i(\tilde{X}(t_{k-1}))\big\vert \mathcal{F}_{k-1} \right) = \dfrac{1}{p-1}\sum_{j \in \mathcal{C}_i,j \ne i}\mathbb{E}(\delta \tilde{K}_{ij} | \mathcal{F}_{k-1}) - \dfrac{1}{N-1} \sum_{j:j \ne i} \mathbb{E}(\delta \tilde{K}_{ij} | \mathcal{F}_{k-1}).
\end{equation*}
Performing Taylor expansion of $\delta \tilde{K}_{ij}$ around the point $t_{k-1}$ to the second order, 
\begin{equation*}
    \delta \tilde{K}_{ij} = \nabla_{x_i} K_{ij}|_{t_{k-1}} \cdot \delta \tilde{X}^i + \nabla_{x_j} K_{ij}|_{t_{k-1}} \cdot \delta \tilde{X}^j + \dfrac{1}{2} M_2 (\delta \tilde{X}^i+\delta \tilde{X}^j) \otimes (\delta \tilde{X}^i+\delta \tilde{X}^j)\cdot\mathbf{1}_d,
\end{equation*}
where $M_2$ is a random variable bounded by $||\nabla^2 K_{ij}||_\infty$. 
Then, by Lemma~\ref{lemma4},
\begin{equation*}
    |\mathbb{E}(\delta \tilde{K}_{ij}|\mathcal{F}_{k-1})| \le C(1+|\tilde{X}^i(t_{k-1})|^q+|\tilde{X}^j(t_{k-1})|^q) \tau.
\end{equation*}

Using Lemma~\ref{lemma5}, we obtain
\begin{equation*}
    \left \| \dfrac{1}{p-1} \sum_{j \in \mathcal{C}_i,j \ne i} \mathbb{E}(\delta \tilde{K}_{ij}|\mathcal{F}_{k-1})  \right \| \le C\tau \left(1+\sqrt{\dfrac{1}{N-1} \sum_{j:j \ne i} \left\| | \tilde{X}^j(t_{k-1})|^{q'} \right\|^2}\right) \le C\tau.
\end{equation*}
And we can derive more easily that
\begin{equation*}
    \left \| \dfrac{1}{N-1} \sum_{j:j \ne i} \mathbb{E}(\delta \tilde{K}_{ij}|\mathcal{F}_{k-1})  \right \| \le C\tau.
\end{equation*}

Therefore, by Lemma~\ref{lemma4},
\begin{equation*}
    \mathbb{E}\left(\chi_i(\tilde{X}(t))-\chi_i(\tilde{X}(t_{k-1}))\big\vert \mathcal{F}_{k-1}\right) \le C
\|Z^i(t_{k-1})\|\tau \le C\|Z^i(t)\|\tau + C\tau^2.
\end{equation*}

Then, by mean value inequality,
\begin{equation*}
    I_1 \le \dfrac{C}{N} \sum_{i=1}^N \|Z^i(t)\|\tau + C\tau^2 \le \eta \dfrac{1}{N} \sum_{i=1}^N \|Z^i(t)\|^2 + C(\eta)\tau^2,
\end{equation*}
where $\eta$ is a constant indpendent of time and $C(\eta)$ is another constant dependent of $\eta$.

Now we estimate $I_2$.
Decompose $I_2$ into
\begin{align*}
    I_2
     = &\dfrac{1}{N} \sum_{i=1}^N \mathbb{E} \left\{ \left(Z^i(t)-Z^i(t_{k-1})\right) \cdot \left(\chi_i(\tilde{X}(t))-\chi_i(X(t))\right) \right\} \\
     & + \dfrac{1}{N} \sum_{i=1}^N \mathbb{E} \left\{ \left(Z^i(t)-Z^i(t_{k-1})\right) \cdot \chi_i(X(t)) \right\} \\
    =&: I_{21} + I_{22}.
\end{align*}
Consider controlling $I_{21}$. Obviously, 
\begin{equation*}
    I_{21} \le \dfrac{1}{N} C\tau \|\chi_i(\tilde{X}(t))-\chi_i(X(t))\|.
\end{equation*}

Note that
\begin{equation} \label{prooftheorem2}
    \|\chi_i(\tilde{X}(t))-\chi_i(X(t))\| \le \left \| \dfrac{1}{p-1} \sum_{j \in \mathcal{C}_i,j \ne i}\delta K_{ij}(t) \right \| + \left \| \dfrac{1}{N-1} \sum_{j:j \ne i}\delta K_{ij}(t) \right \|.
\end{equation}
Because the derivative of $K_{ij}$ is uniformly bounded, there exists a Lipschitz constant $L$ (independent of $i,j$) that
\begin{equation*}
    |\delta K_{ij}(t)| \le L(|Z^i(t)|+|Z^j(t)|).
\end{equation*}

Hence,
\begin{equation}\label{eq22}
    \left \| \dfrac{1}{p-1} \sum_{j \in \mathcal{C}_i,j \ne i}\delta K_{ij}(t) \right \| \le L \left (\|Z^i(t)\| + \left \| \dfrac{1}{p-1} \sum_{j \in \mathcal{C}_i,j \ne i}|Z^j(t)| \right \| \right). 
\end{equation}
By Lemma~\ref{lemma6},
\begin{equation*}
    |Z^i(t)| \le |Z^i(t_{k-1})| + C\tau.
\end{equation*}
Since $Z^i(t_{k-1})$ is dependent of $\mathcal{C}_i$, by Lemma~\ref{lemma5} and Lemma \labelcref{lemma6}, as well as Cauchy inequality, 
\begin{equation} \label{eq23}
    \begin{aligned} 
&  \left \| \dfrac{1}{p-1} \sum_{j \in \mathcal{C}_i,j \ne i}|Z^j(t)| \right \|  \\
\le \ &  \sqrt{\dfrac{1}{N-1}\sum_{j:j \ne i} \| Z^j(t) \|^2} \\
\le \ &  \sqrt{\dfrac{1}{N-1}\sum_{j:j \ne i} (\| Z^j(t_{k-1}) \|+C\tau)^2} \\
= \ &  \ \sqrt{\dfrac{1}{N-1}\sum_{j:j \ne i} \| Z^j(t_{k-1}) \|^2+2C\tau\dfrac{1}{N-1}\sum_{j:j \ne i} \| Z^j(t_{k-1}) \|+C^2\tau^2}                                         \\
\le \  & \ \sqrt{\dfrac{1}{N-1}\sum_{j:j \ne i} \| Z^j(t_{k-1}) \|^2+2C\tau\sqrt{\dfrac{1}{N-1}\sum_{j:j \ne i} \| Z^j(t_{k-1}) \|^2}+C^2\tau^2}   \\
= \ &  \sqrt{\dfrac{1}{N-1} \sum_{j:j \ne i} \|Z^j(t_{k-1})\|^2} + C\tau \ .
\end{aligned}
\end{equation}

By mean value inequality, we obtain
\begin{align*}
    \dfrac{1}{N}\sum_{i=1}^N \tau \sqrt{\dfrac{1}{N-1} \sum_{j:j \ne i} \|Z^j(t_{k-1})\|^2} 
    & \le \eta \dfrac{1}{N} \sum_{i=1}^N \|Z^i(t_{k-1})\|^2 +C(\eta) \tau^2 \\
    & \le \eta \dfrac{1}{N} \sum_{i=1}^N (\|Z^i(t)\|+C\tau)^2 +C(\eta) \tau^2 \\
    & \le \eta \dfrac{1}{N} \sum_{i=1}^N 2(\|Z^i(t)\|^2+C^2\tau^2)+C(\eta) \tau^2 ,
\end{align*}
and
\begin{equation*}
    \dfrac{1}{N}\sum_{i=1}^N \tau \|Z^i(t)\| \le  \eta \dfrac{1}{N} \sum_{i=1}^N \|Z^i(t)\|^2 +C(\eta) \tau^2.
\end{equation*}
Thus, the first term of \eqref{prooftheorem2} is controlled by $\eta \dfrac{1}{N} \sum_{i=1}^N \|Z^i(t)\|^2 + C(\eta)\tau^2$. It is easier to prove that the second term of \eqref{prooftheorem2}  is controlled by $\eta \dfrac{1}{N} \sum_{i=1}^N \|Z^i(t)\|^2 + C(\eta)\tau^2$.

Thus,
\begin{equation*}
    I_{21} \le \eta \dfrac{1}{N} \sum_{i=1}^N ||Z^i(t)||^2 + C(\eta)\tau^2.
\end{equation*}

Finally, we estimate $I_{22}$.
Recall that 
\begin{equation}\label{eq24}
    \dfrac{\mathrm{d}}{\mathrm{d}t}Z^i(t) = \dfrac{1}{p-1} \sum_{j \in \mathcal{C}_i,j \ne i}\delta K_{ij}(t) + \chi_i(X(t)).
\end{equation}
Integrating \eqref{eq24} in time over $[t_{k-1},t)$, then dotting with $\chi_i(X(t))$, taking expectation, we derive
\begin{equation}\label{eq25}
\begin{aligned}
    &\bigl|\mathbb{E} \left\{ (Z^i(t)-Z^i(t_{k-1})) \cdot \chi_i(X(t)) \right\}\bigr| \le \\
    &\int_{t_{k-1}}^t \mathbb{E} \left\{\left( \dfrac{1}{p-1}\sum_{j \in \mathcal{C}_i,j \ne i}\delta K_{ij}(s)\right) \cdot \chi_i(X(t)) \right \} {\rm d}s + \mathbb{E} \left\{ \int_{t_{k-1}}^t \chi_i(X(s)) \cdot \chi_i(X(t)) {\rm d}s \right\}. 
\end{aligned}
\end{equation}
By the estimation of \eqref{eq22} and \eqref{eq23}, the second term on the right hand of \eqref{eq25} is bounded by $C\sqrt{\dfrac{1}{N-1}\sum_{j:j \ne i}\|Z^j(t_{k-1})\|^2}\tau+C\tau^2$. Similarly as the estimation of $I_{21}$, this is less than $\eta \dfrac{1}{N} \sum_{i=1}^N \|Z^i(t)\|^2 + C(\eta)\tau^2$. 

By Lemma~\ref{lemma3}, the first term on the right hand of \eqref{eq25} is bounded by $\bigl(\dfrac{1}{p-1}-\dfrac{1}{N-1}\bigr)\|\Lambda_i\|_\infty \tau$.

Therefore, 
\begin{equation*}
    I_{22} \le \eta \dfrac{1}{N} \sum_{i=1}^N \|Z^i(t_{k-1})\|^2 + C(\eta) \tau^2 + \Bigl(\dfrac{1}{p-1}-\dfrac{1}{N-1}\Bigr)\|\Lambda_i\|_\infty \tau.
\end{equation*}
Note that $\|Z^i(t_{k-1})\|^2 \le (\|Z^i(t)\|^2+C\tau)^2 \le 2\|Z^i(t)\|+2C^2\tau^2$.
Then,
\begin{equation*}
    I_{22} \le \ \dfrac{1}{N} \sum_{i=1}^N \|Z^i(t)\|^2 + C(\eta)\tau^2 + \Bigl(\dfrac{1}{p-1}-\dfrac{1}{N-1}\Bigr)\|\Lambda_i\|_\infty \tau. 
\end{equation*}

Above all, we have an estimate that
\begin{equation}
    \dfrac{\mathrm{d}}{\mathrm{d}t}J(t) \le J(t) + C(\eta)\tau^2 + \Bigl(\dfrac{1}{p-1}-\dfrac{1}{N-1}\Bigr) \|\Lambda_i\|_\infty \tau.
\end{equation}
By Gr$\mathrm{\ddot{o}}$nwall's lemma, Theorem~\ref{} follows.
\hfill $\blacksquare$

\section{Experiments on graph transformers}
In graph processing, especially for node classification tasks, since the input to neural networks is always a complete graph or its subgraph, the data lacks a batch dimension. Therefore, we can effectively leverage the parallel advantages of RBA by applying RBA to the graph transformer, thereby enhancing the model's practicality.
Our strategy involves replacing the self-attention mechanism in existing models with RBA, followed by retraining and inference of the model on the same dataset. We call the new architecture with RBA as RBTransformers.

In this section, we will in two aspects present experimental results that RBTransformers perform better than the corresponding graph transformers in several datasets. One is about the expressivity, that is experimental accuracy reflected in specific downstream tasks. The other is about practicability, that is time saving and memory saving. 

\subsection{Expressivity}
For RBTransformer based on SGFormer, we do node classification tasks on datasets ogbn-arxiv, pokec and ogbn-papers100M. 

We make a brief introduction of these datasets about their numbers of nodes, edges, features, classes, and metric for evaluation, shown in Table\labelcref{dataset}.

\begin{table*}[!htbp]
  \centering
  \caption{Statistics summary of the large-sized graph datasets.}
  \label{dataset}
  \centering
  	\begin{tabular*}{1.05\linewidth}{@{\extracolsep{\fill}\,}lccccc}
 	\toprule
        Dataset  &\#Nodes &\#Edges &\#Features  &\#Classes &Metric \\
        \midrule
        ogbn-arxiv &169,343 &1,166,243 &8 &40 &Accuracy \\
        pokec &1,632,803 & 30,622,564 & 65 & 2 &Accuracy\\
        ogbn-papers100M &110,059,956 &1,615,685,872 &128 &172 &Accuracy \\
    \bottomrule
    \end{tabular*}
\end{table*}

We do the expressivity experiments on ogbn-arxiv and pokec dataset. To indicate our superiority, we make a comparison with the model SGFormer, a powerful graph transformer provided in 2023.\cite{wu2024simplifying}

As shown in Table \labelcref{accuracy}, the accuracy of models with our RBA has little difference with theorigin self-attention mechanism.

\begin{table*}[!htbp]
  \centering
  \caption{Accuracy comparison of the large-sized graph datasets.}
  \label{accuracy}
  \centering
  	\begin{tabular*}{0.92\linewidth}{@{\extracolsep{\fill}\,}lccccc}
 	\toprule
        Dataset  &\#Origin &\#Ours  \\
        \midrule
        ogbn-arxiv &72.63 ± 0.13 &72.90 ± 0.20  \\
        pokec &73.76 ± 0.24 & 75.13 ± 0.14	 \\
    \bottomrule
    \end{tabular*}
\end{table*}

\subsection{Practicability}
We want to compare the propagation time of of the self-attention part of  RBTransformers through doing node-level downstream tasks. To be mentioned, applying the parallelism of the RBA, we run our codes on multiple devices. Unfortunately, due to the communication time between devices, we fail to verify that the self-attention part of RBTransformers is faster than that of origin Transformers.

We also do tasks on huge large dataset ogbn-papers100M, which is easily out of memory with normal graph neural networks and graph transformers.
We run the codes on multiple devices and complete them successfully, which maybe out of memory by the origin method, and record the memory usage on the primary device and other assistant devices. The result is shown in \labelcref{papers100M}.

\begin{table*}[!h]
  \centering
  \caption{Memory usage on ogbn-papers100M}
  \label{papers100M}

  \bigskip
  \centering batch size 1000 \\[0.2em]
  	\begin{tabular*}{1.00\linewidth}{@{\extracolsep{\fill}\,}lccccc}
 	\toprule
        num of devices  &\#8 &\#4 &\#2 &\#1 \\
        \midrule 
        memory usage on primary device &23760MiB & 22450MiB&22280MiB &OOM \\
        memory usage on other devices &3922MiB & 10544MiB & 15108MiB &OOM \\
        \midrule
        test accuracy &65.69 &65.52 &65.87 &OOM  \\
    \bottomrule
    \end{tabular*}
\\
\bigskip
\centering  batch size 1200 \\[0.2em]
    \begin{tabular*}{1.00\linewidth}{@{\extracolsep{\fill}\,}lccccc}
 	\toprule
        num of devices  &\#8 &\#4 &\#2  &\#1\\
        \midrule 
        memory usage on primary device &22724MiB &21874MiB &OOM &OOM  \\
        memory usage on other devices &2838MiB &14352MiB &OOM & OOM\\
        \midrule
        test accuracy &65.30 &65.36 &OOM &OOM  \\
    \bottomrule
    \end{tabular*}
\\
\bigskip
\centering  batch size 1500 \\[0.2em]
    \begin{tabular*}{1.00\linewidth}{@{\extracolsep{\fill}\,}lccccc}
 	\toprule
        num of devices  &\#8 &\#4 &\#2  &\#1 \\
        \midrule
        memory usage on primary device &22962MiB &OOM &OOM &OOM  \\
        memory usage on other devices &5156MiB & OOM & OOM & OOM \\
        \midrule
        test accuracy &65.34 &OOM &OOM &OOM \\
    \bottomrule
    \end{tabular*}
\end{table*}

For larger batch sizes, more graphics cards are often required, and parallalism is even more necessary.
We obtain the maximum batch size that can be processed for each number of devices based on Table \labelcref{papers100M}, as shown in Table \labelcref{tab4}.
\begin{table}[!]
    \centering
    \bigskip
    \bigskip
    \caption{The maximum batch size that the machine can handle}
    \bigskip
    \begin{tabular*}{1.00\linewidth}{@{\extracolsep{\fill}\,}lccccc}
    \toprule
         &  num of devices  &\#1 &\#2 &\#4 &\#8 \\
    \midrule
         &  max batch size  &\ $<1000$ &\ $1000 \sim 1200$ &\ $1200 \sim 1500$ &\ $>1500$ \\
    \bottomrule
    \bigskip
    \bigskip
    \end{tabular*}
    \label{tab4}
\end{table}



\section{Conclusion}
We gain inspiration from Random Batch Methods and put forward a novel attention mechanism, serving as a useful tool for improving existing graph transformer in terms of time and memory saving. We establish a theoretical model for Transformer encoder layer propagation as a small extra mathematical innovation achievements, based on which we give the convergence property. In the experimental part, higher accuracy of node-level downstream tasks on large graph prove the RBTransformers' expressivity, while lower memory requirement prove its practicability. For future work, the further theoretical research from the view of particle system and application of RBA to other powerful Transformer models is worth doing. Moreover, the verification of time saving ability of RBA is an important experiment in the next research stage.




\acks{We would like to acknowledge Lei Li, which is one of the Random Batch Methods presenters and his instruction on the theoretical part in this paper. }


\newpage

\appendix

\section{Experiment details}

\noindent

We execute our experimental tasks on simple devices, through the method 'Pytorch- DataParallel'. Most experiments for expressivity and practicability are finished on eight Nvidia 4090 devices, minority are finished on six Nvidia A-100 devices and we do not try to repeat these experiments on 4090 devices to prove they can bot be performed on lower-level GPUs.

\section{Brief talk about future work}

\noindent

For future work, we currently have three plans. First, go on doing research on the theoretical property about RBA, based on particle system. The second is to do further applications, promote RBA to other graph transformer models and even to models on natural language processing and protein design, etc. The last is trying more advanced parallel implementation of RBA on multiple devices, as the ability of RBA to save time is not verified up till now because of the communication time between devices. Moreover, the permutation invariance of RBA is worth researching, which is also a problem of several original graph transformer models with which RBA is combined.

\vskip 0.2in
\newpage

\bibliography{AJI}

\end{document}